\documentclass[letterpaper, 10 pt, conference]{ieeeconf}  

\IEEEoverridecommandlockouts                              

\usepackage{booktabs}
\usepackage{multirow}
\usepackage{graphicx}
\usepackage{placeins}
\usepackage{amsmath} 
\usepackage{amssymb}  
\usepackage{url}
\usepackage{cite}
\usepackage{dblfloatfix}

\graphicspath{{\subfix{../images/}}}

\makeatletter
\let\NAT@parse\undefined
\makeatother

\usepackage{hyperref}
\usepackage{siunitx}
\usepackage{subfiles}

\hypersetup{hidelinks}
\newtheorem{rem}{Remark}

\newtheorem{prop}{Proposition}

\newtheorem{assum}{Assumption}

\newcommand\tran{\mkern-2mu\raise1.25ex\hbox{$\scriptscriptstyle\top\hspace{0.5mm}$}\mkern-3.5mu}
\newcommand{\R}{\mathbb{R}}

\newcommand{\C}{\mathcal{C}}

\newcommand{\Z}{\mathcal{Z}}
\newcommand{\bm}[1]{{\boldsymbol{#1}}}

\title{\LARGE \bf
Structure-Aware Spatial Kernels for Gaussian Process Distributed Port-Hamiltonian Systems
}

\author{Tianyi Yang$^{1}$ and Thomas Beckers$^{2}$  
\thanks{$^{1}$ is with the Department of Computer Science, Vanderbilt University, Nashville, TN 37212, USA {\tt\small tianyi.yang@Vanderbilt.edu}} %
\thanks{$^{2}$ is with the Institute for Measurement, Control and Microtechnology, Ulm University, 89081 Ulm, Germany {\tt\small thomas.beckers@uni-ulm.de}}
}

\begin{document}

\maketitle
\thispagestyle{empty}
\pagestyle{empty}

\begin{abstract}

Gaussian process distributed port-Hamiltonian systems (GP-dPHS) enable data-driven modeling of distributed dynamical systems while preserving their physical structure. However, standard stationary kernels can be inefficient for spatially heterogeneous systems because they neither exploit known input invariances nor reflect structured spatial dependencies. We address this limitation with a structure-aware kernel design that removes irrelevant state coordinates and decouples state and spatial dependencies through a tensor-product covariance. This modular construction allows qualitative knowledge of spatial heterogeneity to be incorporated by selecting an appropriate spatial kernel without requiring an analytical description of the underlying profile. Numerical experiments on a nonlinear string with periodic, smoothly varying, and semi-piecewise stiffness profiles demonstrate that structural decoupling consistently improves reconstruction accuracy, while profile-matched spatial kernels yield substantially larger gains in force-field and trajectory reconstruction. 

\end{abstract}

\section{Introduction} \label{sec:introduction}
Data-driven modeling has become increasingly important in control, robotics, and physical-system identification, particularly when first-principles models are incomplete or computationally expensive~\cite{dataDrivenControl}.
In these applications, learned models should not only be able to predict trajectories but also remain consistent with the physical structure of the underlying system, as this consistency supports trustworthy predictions and reliable use in control of safety-critical systems.
Physics-informed machine learning addresses this requirement by incorporating prior knowledge of conservation laws, constitutive relations, and system interconnections into the learning process~\cite{PINN}.

Recently, port-Hamiltonian systems (PHS) have been increasingly used as an energy-based modeling approach for nonlinear system identification~\cite{PHSIntro}. They provide a natural framework for structure-preserving modeling in control by representing dynamics in terms of energy storage, interconnection, dissipation, and external ports.
This framework has motivated several data-driven modeling approaches, including port-Hamiltonian neural networks that learn energy functions, dissipation, and external forcing, and operator-inference methods that identify structure-preserving reduced-order operators from data and a prescribed Hamiltonian form~\cite{desai2021phnn, geng2025phopinf}.
Within this broader class of approaches, Gaussian process distributed port-Hamiltonian systems (GP-dPHS) combine the port-Hamiltonian representation for partial differential equations (PDEs) with nonparametric probabilistic learning, enabling unknown energy-related quantities to be inferred from data while retaining an interpretable physical structure~\cite{GPPHSPaper, tan24a}.
This Bayesian formulation provides predictive uncertainty estimates, while the distributed port-Hamiltonian representation preserves the compositional structure of the physical system. 

Existing GP-dPHS formulations typically employ squared-exponential (RBF) kernels to learn the Hamiltonian or its associated dynamics~\cite{tan24a}.
Although such kernels provide flexible general-purpose priors, their stationary, distance-based similarity structure can be restrictive in spatial modeling problems whose dependence patterns vary across the domain~\cite{stationaryKernelLimitation}.

This limitation is particularly relevant for spatially heterogeneous continuum systems, which are ubiquitous in modern engineering and biological applications.
For instance, soft robotic manipulators may employ periodically spaced compliant sections to promote prescribed bending modes.
Additionally, piecewise stiffness profiles naturally emerge in multi-material composite structures, while $C^2$-smooth profiles provide an idealized representation of smoothly graded materials~\cite{FGMexample} and heterogeneous biological tissues~\cite{bioTissueEg}. In such systems, spatial locations can exhibit fundamentally different or recurring constitutive behavior despite belonging to the same physical domain. Consequently, an effective learning model should capture these spatial relationships rather than relying solely on geometric proximity to determine similarity.

The objective of this paper is to improve GP-dPHS modeling of spatially heterogeneous systems by incorporating available structural knowledge into the covariance design. In particular, we consider both known physical invariances and coarse information about spatial dependence.
To illustrate the limitations of standard kernel choices, we consider a one-dimensional string with spatially varying stiffness, constant damping, and localized motor forcing. While a stationary squared-exponential kernel provides nearly exact reconstruction for spatially uniform stiffness, its performance deteriorates substantially for heterogeneous profiles. In contrast, incorporating spatial priors aligned with the qualitative stiffness structure substantially improves reconstruction performance.

\textbf{Contribution:}
This paper proposes a systematic, modular framework for incorporating known physical invariances and qualitative spatial knowledge into GP-dPHS kernel design. The framework removes irrelevant state coordinates from individual effort models and decouples the remaining state and spatial dependencies through a tensor-product covariance. This separation provides a systematic interface for selecting spatial kernels according to coarse structural information, such as periodicity, finite smoothness, or transitions between distinct spatial regimes, without requiring an analytical description of the underlying heterogeneous profile. In this way, prior knowledge is encoded directly in the covariance while preserving the underlying port-Hamiltonian structure. The approach is demonstrated on a nonlinear string with heterogeneous stiffness, where profile-matched spatial kernels substantially improve force-field and trajectory reconstruction over a stationary RBF baseline.

The remainder of this paper is organized as follows. Section~\ref{sec:background_problem_formulation} introduces the GP-dPHS framework and learning problem. Section~\ref{sec:kernel} presents the proposed structure-aware kernel design, and Section~\ref{sec:string_system_and_result} provides its numerical evaluation on a nonlinear string system.

\section{Background and Problem Formulation} \label{sec:background_problem_formulation}
In this section, we briefly introduce distributed Port-Hamiltonian systems and Gaussian Process distributed Port-Hamiltonian systems (GP-dPHS).
The latter provides a Bayesian framework for learning partially unknown PDE dynamics while preserving the underlying physical structure.

\subsection{Distributed Port-Hamiltonian Systems}
\label{subsec:dphs}
Distributed Port-Hamiltonian systems (dPHS) provide a systematic framework for representing physical systems governed by partial differential equations.
In contrast to finite-dimensional Port-Hamiltonian systems, the interconnection, dissipation, and input operators can contain spatial differential operators.
The dynamics are characterized by a Hamiltonian functional representing the total energy of the distributed system.

Let\footnote{Vectors~$\bm a$ and vector-valued functions~$\bm f(\cdot)$ are denoted with bold characters. Matrices are described with capital letters. $I_n$ is the $n$-dimensional identity matrix and $0_n$ the zero matrix. The expression~$A_{:,i}$ denotes the i-th column of $A$. For a positive definite matrix $\Lambda$, the associated squared norm is $\|\bm{x} - \bm{y}\|_{\Lambda}^2 = (\bm{x} - \bm{y})^\top \Lambda (\bm{x} - \bm{y})$.  $\R_{>0}$ denotes the set of positive real numbers, while $\R_{\geq 0}$ is the set of non-negative real numbers. $\C^i$ denotes the class of $i$ times continuously differentiable functions.} $\mathcal{Z}\subset\mathbb{R}^{d}$ denote a compact spatial domain, and let $\bm{x}(t,\bm{z})\in\mathbb{R}^{n}$ denote the state at time $t$ and position $\bm{z}$.
For each fixed $t$, the distributed state $\bm{x}(t,\cdot)$ belongs to a suitable space $\mathcal{X}$ of vector-valued functions on $\mathcal{Z}$.
The Hamiltonian functional $\mathcal{H} \colon \mathcal{X} \to \mathbb{R}$ is defined as
\begin{equation}
    \mathcal{H}[\bm{x}(t,\cdot)] = \int_{\mathcal{Z}} H(\bm{z}, \bm{x}(t,\bm{z})) \,\mathrm{d}\bm{z},
    \label{eq:hamiltonian_functional}
\end{equation}
where $H \colon \mathcal{Z}\times\mathbb{R}^{n}\to\mathbb{R}$ denotes the local Hamiltonian density.
The corresponding variational derivative is denoted by $\delta_x\mathcal{H}$.
Following the general dPHS formulation, the dynamics can be written as
\begin{equation}
    \frac{\partial \bm{x}}{\partial t} = \left(\mathcal{J}-\mathcal{R}\right) \delta_x\mathcal{H} + \mathcal{G}\bm{u},
    \label{eq:dphs}
\end{equation}
with output $\bm{y}\in\R^m$ defined by $\bm{y} = \mathcal{G}^{*}\delta_x\mathcal{H}$.
Here, $\mathcal{J}$ is a skew-adjoint differential operator describing the power-conserving interconnection structure, with its boundary terms accounting for boundary power exchange.
The dissipation operator $\mathcal{R}$ is self-adjoint and positive semidefinite on its specified domain, while $\mathcal{G}$ describes coupling to the external input $\bm{u}$.
The operator domains and boundary conditions are chosen consistently with the associated energy balance.

The Port-Hamiltonian representation separates the physical interconnection structure from the Hamiltonian.
In particular, $\mathcal{J}$ and $\mathcal{R}$ determine how energy is exchanged and dissipated, whereas $\mathcal{H}$ characterizes the constitutive behavior of the system.
This separation is particularly useful for learning physical systems since complicated and potentially unknown nonlinearities can be represented through the Hamiltonian while retaining the known physical structure of the PDE.
A more detailed introduction can be found in~\cite{rashad2020twenty}.

\subsection{Gaussian Process Distributed Port-Hamiltonian Systems} \label{subsec:gpdphs}
For many physical PDE systems, the structure of the differential operators in~\eqref{eq:dphs} can be derived from first principles, whereas the Hamiltonian functional is only partially known.
This situation typically arises in highly nonlinear systems for which structural knowledge is available, but the exact governing equations are difficult to obtain.
An example is soft robotic systems, where complex and hard-to-characterize material behavior can give rise to unknown nonlinear dynamics.
Gaussian Process distributed Port-Hamiltonian systems (GP-dPHS) address this problem by learning the unknown Hamiltonian from data while preserving the dPHS structure.
We provide a brief introduction here, and more details can be found in~\cite{tan24a}.

Consider observations of the distributed system at temporal instants $t_i\geq 0$ and spatial locations $\bm{z}_j\in\Z$,
\begin{equation}
    \mathcal{D} = \left\{ t_i, \bm{z}_j, \bm{x}(t_i,\bm{z}_j), \bm{u}(t_i) \right\}_{i=1,j=1}^{N_t,N_z}.
    \label{eq:gpdphs_dataset}
\end{equation}
The observations can first be used to estimate the state derivatives and, if necessary, to obtain a denser spatial representation of the state.
This results in an augmented training dataset containing states and their corresponding time derivatives, $\mathcal{E} = \left\{X,\dot{X}\right\} $, where $X$ denotes the spatially stacked states and $\dot{X}$ their estimated temporal derivatives.

The central idea of GP-dPHS is to place a Gaussian Process prior on the unknown Hamiltonian.
Since Gaussian Processes are closed under linear transformations~\cite{jidling2017linearly}, the differential operator defining the dPHS dynamics can be applied directly to the GP.
Consequently, the dynamics in~\eqref{eq:dphs} admit the probabilistic representation
\begin{equation}
    \frac{\partial \bm{x}}{\partial t} \sim \mathcal{GP}\!\big(\mathcal{G}\bm{u},\, k_{\mathrm{dPHS}}(\bm{x}, \bm{x}') \big).
    \label{eq:gpdphs_gp}
\end{equation}
Here, $k_{\mathrm{dPHS}}$ is a PHS-constrained covariance function, and the operators may depend on a finite set of unknown physical parameters $\Theta$ in the general GP-dPHS formulation.
In the present study, the physical operators and their parameters are treated as known, as specified in Assumption~\ref{ass:known_structure}.

Starting from a squared-exponential covariance function for the Hamiltonian
\begin{equation}
    k_H(\bm{x},\bm{x}') = \sigma_f^2 \exp\!\left( -\frac{\|\bm{x}-\bm{x}'\|^2}{2\phi_l^2} \right),
    \label{eq:hamiltonian_kernel}
\end{equation}
the PHS-constrained covariance function $k_{\mathrm{dPHS}}$ is obtained by applying the Port-Hamiltonian operators to the covariance of the Hamiltonian~\eqref{eq:hamiltonian_kernel}, resulting in
\begin{equation}
    \begin{split}
        k_{\mathrm{dPHS}}(\bm{x},\bm{x}') &= \left(\mathcal{J} -\mathcal{R}\right) \delta_x k_H(\bm{x},\bm{x}') \delta_{x'}^{\top} \left(\mathcal{J} -\mathcal{R}\right)^{\top}.
    \end{split}
    \label{eq:gpdphs_kernel}
\end{equation}
Thus, the physical structure is incorporated directly into the covariance function rather than imposed after learning.
The potentially unknown kernel hyperparameters $\phi_l$ and $\sigma_f$ can be optimized by minimizing the negative log marginal likelihood.

After conditioning on the observations $\mathcal{E}$, the learned dynamics can equivalently be written in distributed Port-Hamiltonian form as
\begin{equation}
    \frac{\partial \hat{\bm{x}}}{\partial t} = \left( \mathcal{J} - \mathcal{R} \right) \delta_x \widehat{\mathcal{H}} + \mathcal{G}\bm{u},
    \label{eq:gpdphs_dynamics}
\end{equation}
where $\widehat{\mathcal{H}}$ is endowed with the GP posterior distribution.
Hence, the GP-dPHS does not represent only a single learned vector field, but a distribution over distributed Port-Hamiltonian models.
An important consequence of this construction is that a sample $\widehat{\mathcal{H}}$ from the GP posterior is a deterministic Hamiltonian functional.
Substituting such a sample into~\eqref{eq:gpdphs_dynamics} therefore yields a dPHS with the prescribed interconnection and dissipation structure.
As a result, the sampled models retain the physical energy structure of the underlying dPHS while the GP posterior provides a quantitative description of the remaining model uncertainty.

\subsection{Problem Setting}
\label{subsec:heterogeneous_problem}
We consider the problem of learning partially unknown constitutive relations of a \emph{spatially heterogeneous} PDE system whose dynamics admit the distributed Port-Hamiltonian representation~\eqref{eq:dphs} via GP-dPHS.
Spatial heterogeneity may arise, for example, in soft robotic manipulators with geometric variations or composed of multiple materials with different stiffnesses.
Throughout the paper, we assume that the PDE system~\eqref{eq:dphs}, with prescribed initial and boundary conditions, has a unique solution on the time interval considered, with sufficient temporal and spatial regularity for the derivatives used in the formulation.
The Hamiltonian functional $\mathcal{H}$ is assumed to be unknown due to, for example, nonlinear stress–strain relations.
We also assume $\mathcal{H}$ to be infinitely Fr\'echet differentiable with respect to the state on an appropriate state-function space.
The objective is to obtain predictions $\hat h_i\colon\R^n\times\mathcal{Z}\to\R$ such that the learned model
\begin{align}
    &\frac{\partial \hat{\bm{x}}}{\partial t} = \left(\mathcal{J}-\mathcal{R}\right) \hat{\bm{h}}(\hat{\bm{x}}, \bm{z}) + \mathcal{G}\bm{u}, \quad \hat{\bm{h}} = \begin{bmatrix} \hat h_1 & \cdots & \hat h_n \end{bmatrix}^{\top},\notag \\
    &\text{with } \hat h_i \approx h_i(\bm{x}, \bm{z}) := \frac{\delta \mathcal{H}}{\delta x_i}(\bm{x}, \bm{z}), \quad i = 1, \ldots, n.
    \label{eq:heterogeneous_learned_dphs}
\end{align}
closely approximates the actual system dynamics in~\eqref{eq:dphs}.
Building on GP-dPHS in~\cite{tan24a}, we especially focus on using coarse prior information about spatially varying constitutive properties to construct structure-aware covariance functions.

\begin{rem}
In particular, we do not assume that an analytical description of the spatially varying constitutive law with respect to $\bm{z}$ is available.
Instead, we have access to coarse structural information about its spatial dependence.
For example, this information may indicate that the constitutive properties vary periodically across the spatial domain, possess finite regularity, or transition between distinct spatial regimes as in a soft robot with varying geometry or materials.
\end{rem}

\subsection{Assumptions}
To formalize the learning problem, we introduce the following assumptions.

\begin{assum}
\label{ass:known_structure}
The differential operators $\mathcal{J}$, $\mathcal{R}$, and $\mathcal{G}$ defining the dPHS structure~(\ref{eq:dphs}) are known.
\end{assum}

\begin{assum}
\label{ass:observations}
The state of the PDE system~\eqref{eq:dphs} can be observed at temporal instants $t_i$ and spatial locations $\bm{z}_j$, yielding the set of observations
$
    \left\{ \bm{x}(t_1,\bm{z}_1),\!\ldots\!, \bm{x}(t_1,\bm{z}_{N_z}),\! \ldots\!, \bm{x}(t_{N_t},\bm{z}_1),\ldots, \bm{x}(t_{N_t},\bm{z}_{N_z}) \right\}\!.
$
\end{assum}

\begin{assum}
\label{ass:separable_hamiltonian}
The true Hamiltonian density is separable with respect to the scalar state components, taking the form $H(\bm{z}, \bm{x}) = \sum_{i=1}^{n} H_i(\bm{z}, x_i)$.
Consequently, each effort component $h_i$ depends only on $x_i$ and $\bm{z}$ and is independent of $x_j$ for $j \neq i$.
\end{assum}

\begin{assum}
\label{ass:heterogeneous_structure}
Coarse prior information about the spatial dependence of the unknown effort components $h_i(x_i, \bm{z})$ is available.
\end{assum}
Assumption~\ref{ass:known_structure} is only mildly restrictive since $\mathcal{J}$ and $\mathcal{G}$ are typically matrix differential operators whose structure is determined by prior knowledge of the PDE system to be learned.
However, this assumption can be further relaxed by treating unknown physical parameters as hyperparameters during the learning process, see~\cite{tan24a} for details.
Assumption~\ref{ass:observations} ensures that data can be collected from the PDE system and therefore requires observability of the state.
If the state cannot be directly observed, an appropriate observer must be employed.
Assumption~\ref{ass:separable_hamiltonian} specifies the state invariances used in the covariance design, while Assumption~\ref{ass:heterogeneous_structure} provides qualitative information for selecting the spatial covariance.
\begin{rem}
Assumption~\ref{ass:separable_hamiltonian} can be weakened by allowing separability between groups of state variables rather than between individual scalar components.
Let $\{I_b\}_{b=1}^{B}$ be a partition of the state indices and consider a Hamiltonian density of the form $H(\bm{z}, \bm{x}) = \sum_{b=1}^{B} H_b(\bm{z}, \bm{x}_{I_b})$.
For $i \in I_b$, the effort component $h_i$ may depend on all state variables in $\bm{x}_{I_b}$ but remains independent of variables outside that group.
For example, a kinetic--potential decomposition $H(\bm{z}, \bm{p}, \bm{q}) = T(\bm{z}, \bm{p}) + V(\bm{z}, \bm{q})$ preserves independence between momentum and configuration variables without requiring separability within either group.
However, independence between groups does not guarantee integrability among independently learned effort components within the same group and additional compatibility constraints are required to ensure that.
\end{rem}

\section{Structured Kernel Design for GP-dPHS} \label{sec:kernel}
Before we present the proposed structure-aware kernel in Subsection~\ref{sec:structure_aware_kernel}, we first show the trivial extension of the GP-dPHS model and the resulting deficits.

\subsection{Trivial Extension}
\label{subsec:baseline_model}
For a system with heterogeneous spatial structure, a straightforward modeling approach is to extend the input space of the GP-dPHS kernel~\eqref{eq:gpdphs_kernel} to include the spatial variable $\bm{z}$ associated with that structure.
Thus, the input is extended to $\boldsymbol{\xi} = [\bm{x}^\top, \bm{z}^\top]^\top$ and used in the squared-exponential Hamiltonian covariance in \eqref{eq:hamiltonian_kernel} as the augmented input space.
This covariance is assigned directly to each independently trained effort model.
By assigning a separate lengthscale to each coordinate, the trivial extended PHS kernel is given by
\begin{align}
    k_{\mathrm{triv}}(\boldsymbol{\xi}, \boldsymbol{\xi}') = \sigma_f^2 \exp\left[-\frac{1}{2} \sum_{j=1}^{n+d} \frac{(\xi_j - \xi_j')^2}{\ell_j^2}\right]. \label{for:trivial}
\end{align}
The coordinate-specific lengthscales $\ell_j$ control covariance decay along individual input dimensions and provide the automatic relevance determination (ARD) mechanism~\cite{ARDpaper}.
Large lengthscales reduce sensitivity to variation in the corresponding coordinates, allowing the model to learn approximate input irrelevance.

This trivial extension would learn each effort component $h_i$ independently using the complete input vector $\boldsymbol{\xi}$, without explicitly imposing known invariances or distinguishing the roles of state and spatial variables.
In addition, despite the ARD coordinate-dependent scaling, the covariance remains stationary because it depends only on the input difference $\boldsymbol{\xi} - \boldsymbol{\xi}'$, see~\cite[Sec.~2.1]{stationaryKernelPaper}.
Consequently, the same distance-based similarity rule is applied throughout the input domain.
Even though this kernel is universal on compact input domains, meaning that continuous effort functions can be approximated arbitrarily closely by functions from its reproducing kernel Hilbert space (RKHS)~\cite{universalKernel}, the approximation property does not ensure efficient learning from limited data.
A generic ARD-RBF covariance does not directly encode known invariances, recurring structure, or localized transitions between regimes.
For example, observations separated by one period along a spatial coordinate may represent equivalent physical conditions, while their RBF covariance decreases with separation.
In conclusion, the trivial extension $k_{\mathrm{triv}}$ in~\eqref{for:trivial} must therefore infer these relationships from observations, which can be inefficient, especially when training data are sparse or unevenly distributed.

\subsection{Structure-Aware Kernel}
\label{sec:structure_aware_kernel}
The trivial extension in the previous subsection treats all state and spatial coordinates uniformly, although prior knowledge may already specify which state variables are relevant to a particular effort component and provide qualitative information about its spatial dependence.
Under Assumptions~\ref{ass:known_structure}--\ref{ass:heterogeneous_structure}, we exploit these two sources of prior information directly in the covariance design.
In particular, for the $i$-th effort component $h_i$, Assumption~\ref{ass:separable_hamiltonian} implies that $h_i$ depends only on the associated state component $x_i$ and the spatial coordinate~$\bm{z}$.
We therefore remove the irrelevant state coordinates $x_j$, $j\neq i$, from the GP input and separate the remaining state and spatial dependencies through the product kernel
\begin{equation}
    k_i\big((x_i, \bm{z}),(x_i', \bm{z}')\big) = k_{x_i}(x_i,x_i')\,k_z(\bm{z}, \bm{z}'),
    \label{eq:structure_aware_kernel}
\end{equation}
where $k_{x_i}$ models similarity with respect to the relevant state component and $k_z$ models similarity with respect to the spatial coordinate.
The latter can be chosen according to the coarse structural information available from Assumption~\ref{ass:heterogeneous_structure}, without requiring an analytical description of the spatially varying constitutive law.

\begin{prop}
\label{prop:structure_aware_kernel}
Let $k_{x_i}$ and $k_z$ be positive semidefinite kernels on $\mathcal{X}_i\subseteq\mathbb{R}$ and $\mathcal{Z}$, respectively, with corresponding reproducing kernel Hilbert spaces $\mathcal{H}_{x_i}$ and $\mathcal{H}_z$.
Then the covariance function in~\eqref{eq:structure_aware_kernel} is a positive semidefinite kernel on $\mathcal{X}_i\times\mathcal{Z}$ and its associated RKHS is the Hilbert tensor product
\begin{equation}
    \mathcal{H}_i = \mathcal{H}_{x_i}\otimes\mathcal{H}_z.
\end{equation}
In particular, $\mathcal{H}_i$ is the completion of finite sums of the form
\begin{equation}
    f(x_i,\bm{z}) = \sum_{r=1}^{R} \varphi_r(x_i)\psi_r(\bm{z}), \quad \varphi_r\in\mathcal{H}_{x_i}, \psi_r\in\mathcal{H}_z.
    \label{eq:tensor_sum}
\end{equation}
Consequently, the product covariance does not restrict the learned effort to a single separable function of the form $f(x_i)\psi(\bm{z})$.
\end{prop}

\begin{proof}
Consider arbitrary points $\{(x_i^{(j)}, \bm{z}^{(j)})\}_{j=1}^{N}$. 
Let $K_x$ and $K_z$ denote the Gram matrices generated by $k_{x_i}$ and $k_z$, respectively.
Since both kernels are positive semidefinite, $K_x\succeq0$ and $K_z \succeq 0$.
The Gram matrix associated with \eqref{eq:structure_aware_kernel} is given by
\begin{equation}
    K_i = K_x\circ K_z,
\end{equation}
where $\circ$ denotes the Hadamard product.
By the Schur product theorem, $K_i\succeq0$, and hence $k_i$ is a valid positive semidefinite kernel.

For product kernels, the corresponding RKHS is the Hilbert tensor product of the RKHSs associated with the individual factors,
\begin{equation}
    \mathcal{H}_{i} = \mathcal{H}_{x_i}\otimes\mathcal{H}_z,
\end{equation}
see~\cite[Theorem~3.4]{productOfKernel} for details.
By definition, this space is obtained as the completion of finite linear combinations of elementary tensor-product functions $\varphi_r(x_i)\psi_r(\bm{z})$, which gives~\eqref{eq:tensor_sum}.
Thus, although the covariance itself factorizes into a state and a spatial kernel, functions represented by its RKHS may contain general finite sums of interactions between state and spatial coordinates and are not restricted to a single separable product.
\end{proof}

\begin{rem}[Architectural modularity]
\label{rem:architectural_modularity}
An ARD-RBF kernel on the reduced input $(x_i,z)$ already admits an algebraic factorization into an RBF kernel on $x_i$ and an RBF kernel on $\bm{z}$, up to the overall output scale.
Hence, writing the covariance explicitly in the product form~\eqref{eq:structure_aware_kernel} does not by itself enlarge the model class when both factors are RBF kernels.
Its benefit is instead \emph{architectural modularity}.
The state covariance $k_{x_i}$ and spatial covariance $k_z$ become independent design components, so that $k_z$ can be replaced by a kernel reflecting known heterogeneous structure without modifying the state-dependent part of the model.
This separation therefore provides a systematic interface for incorporating coarse spatial prior knowledge while leaving the analytical form of the unknown constitutive relation to be learned from data.
\end{rem}

\subsection{Spatial Kernel Choices}
\label{subsec:spatial_kernels}
A compatible spatial kernel can be selected using qualitative structural information inferred from preliminary observations, without requiring a detailed reconstruction of the underlying spatial profile.
For example, a periodic kernel can encode recurring spatial patterns, a Mat\'ern-$5/2$ kernel can represent finitely smooth spatial variation, and a changepoint kernel can blend distinct covariance models across a transition between spatial regimes.
Parameters such as the fundamental period, changepoint location, and transition scale may be initialized from coarse estimates and subsequently optimized during GP training.
The string benchmark in Section~\ref{sec:string_system_and_result} applies these kernel choices to spatially varying stiffness. Furthermore, we observe that the learned changepoint parameters and predictive performance in that benchmark remain insensitive to initialization errors.

\section{Numerical Evaluation} \label{sec:string_system_and_result}
We evaluate the proposed structure-aware kernel design using a nonlinear string with heterogeneous stiffness.
We first describe the physical system and data-generation procedure, followed by model training, trajectory reconstruction, and the numerical results.

\subsection{String Model}
\label{sec:modeling}
We consider a one-dimensional string of length $L=10$, with transverse displacement $u(z,t)$ at position $z \in [0,L]$ and time $t\geq 0$.
Under the unit-density assumption, the momentum density coincides with the transverse velocity, $p(z,t) = \partial u / \partial t$, while the strain is $q(z,t) = \partial u / \partial z$.
With the state ordering $\bm{x} = [q,p]^\top$, we consider the ground-truth Hamiltonian functional given by
\begin{gather}
    \label{eq:string_hamiltonian}
    \mathcal{H}[\bm{x}] = \int_0^L H(z, \bm{x}) \,\mathrm{d}z, \notag\\
    \text{with} \quad H(z, \bm{x}) = \frac{1}{2} p^2 + \frac{1}{2} k(z) q^2 + \frac{1}{2} q^4,
\end{gather}
where $k(z)$ is the spatially varying stiffness and the quartic term introduces nonlinear strain stiffening.
The analytical Hamiltonian generating reference data is \emph{unknown} to the GP models, which learn its effort components independently.
The corresponding effort functions are $h_1 := \delta \mathcal{H} / \delta q = \partial H / \partial q = k(z) q + 2 q^3$ and $h_2 := \delta \mathcal{H} / \delta p = \partial H / \partial p = p$. Thus, the dynamics take the distributed port-Hamiltonian form
\begin{equation}
    \begin{bmatrix} \dot q \\[1mm] \dot p \end{bmatrix} = \begin{bmatrix} 0 & \partial_z \\[1mm] \partial_z & -\gamma \end{bmatrix} \begin{bmatrix} h_1 \\[1mm] h_2 \end{bmatrix} + \begin{bmatrix} 0 \\[1mm] F_{\mathrm{ext}}(z,t) \end{bmatrix},
    \label{eq:string_dynamics}
\end{equation}
where $\gamma = 0.05$ is the internal damping coefficient.
The external input is a localized periodic motor pulse applied near the center of the string with $F_{\mathrm{ext}}(z,t) = \exp[-10 (z - L/2)^2] \sin(2 \pi t)$. The initial conditions are $p(z,0) = 0$ and
\begin{gather}
    q(z,0) = -4 (z - z_c) \exp\!\left[-(z - z_c)^2\right], \; z_c = L/2.
    \label{eq:string_initial_conditions}
\end{gather}
The boundary conditions are given by the homogeneous Dirichlet conditions
$p(0,t) = p(L,t) = 0$ for all $t\geq 0$, while $q$ is left unconstrained at the boundary.

Motivated by the applications discussed in Section~\ref{sec:introduction}, we consider three heterogeneous stiffness profiles: periodic, $C^2$ bumps, and semi-piecewise. Figure~\ref{fig:stiffness_profiles} illustrates the three heterogeneous profiles. The transition in the semi-piecewise profile is $C^2$ smooth rather than strictly discontinuous, so that the spatial derivatives remain well-defined across the interface for finite-difference simulation and trajectory reconstruction. A constant-stiffness experiment provides a spatially homogeneous reference case.

\begin{figure}[t]
    \centering
    \includegraphics[width=0.8\linewidth]{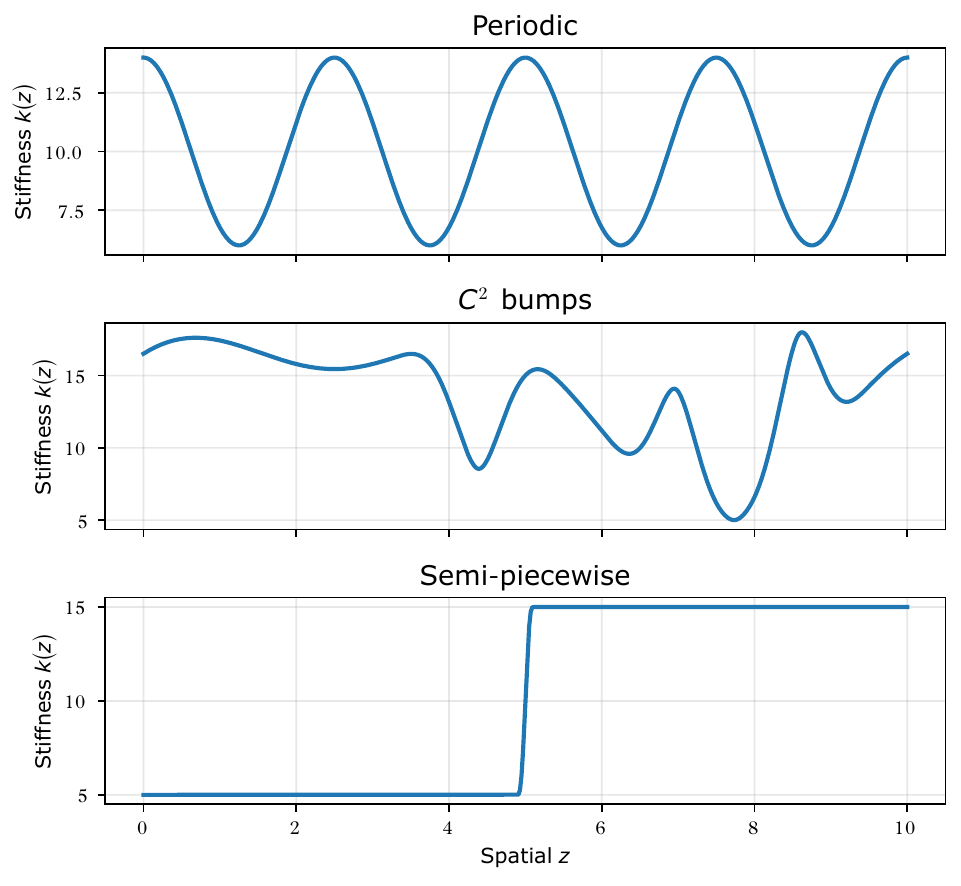}
    \vspace{-0.3cm}\caption{Representative spatial stiffness profiles used in this study: periodic, $C^2$ bumps, and semi-piecewise stiffness (top to bottom, respectively).}
    \vspace{-0.6cm}
    \label{fig:stiffness_profiles}
\end{figure}

\subsection{Simulation and Data Generation}
\label{sec:dataset}
For the generation of training data, the dynamics~\eqref{eq:string_dynamics} are discretized on a uniform spatial grid with $N_z=100$ points.
Let $D_z$ denote the discrete spatial derivative operator on the uniform grid, with $f_i = f(z_i)$.
At interior grid points, it uses the second-order centered-difference approximation
\begin{equation}
    [\partial_z f]_{z_i} \!\approx\! [D_z f]_i := \frac{f_{i+1} - f_{i-1}}{2 \Delta z}, \,i = 2,\ldots,N_z-1.
    \label{eq:string_spatial_difference}
\end{equation}
At the endpoints, $D_z$ uses first-order one-sided differences
$
    [D_z f]_1 := \frac{f_2 - f_1}{\Delta z}$ and $ [D_z f]_{N_z} := \frac{f_{N_z} - f_{N_z-1}}{\Delta z}.
$
The endpoint momentum values are maintained at zero throughout integration, while no boundary values are prescribed for the strain.
The resulting system of ordinary differential equations is integrated using the adaptive Runge--Kutta method RK45, and the numerical solution is recorded at uniformly spaced output times $\Delta t = 0.05\,\mathrm{s}$ for a simulation time of $T = 5.5\,\mathrm{s}$.

Each training sample consists of the augmented input $\boldsymbol{\xi}_{ij} = [p(z_i,t_j), q(z_i,t_j), z_i]^\top$ paired with the corresponding values\footnote{This setup isolates the effect of kernel design from errors associated with reconstructing effort targets from trajectory observations.
In practice, the values of $h_1$ and $h_2$ would be reconstructed from measurements of $\dot{q}$ and $\dot{p}$ as described in~\cite{tan24a}.} of $h_1$ and $h_2$.
From the complete space--time dataset, we select 2\% of the samples for training.
Spatially stratified sampling allocates 25\%, 50\%, and 25\% of the training budget to the left-end region $z \in [0, 2]$, interior region $z \in [2, 8]$, and right-end region $z \in [8, 10]$, respectively.
Within each stratum, samples are selected uniformly at random without replacement.
For each heterogeneous stiffness profile, the same fraction of the available space--time samples is selected.

\subsection{Experimental Design and Model Configurations}
\label{subsec:model_configurations}
The experiments distinguish two contributions of the proposed methodology: imposing known state invariances and incorporating qualitative information about the spatial stiffness profile.
We use two independently trained scalar-output GP models, $\hat h_1 \approx \delta \mathcal{H} / \delta q$ and $\hat h_2 \approx \delta \mathcal{H} / \delta p$, for the constitutive and momentum efforts, respectively.
For each stiffness profile and sampling seed, all three configurations use the same training observations.

The Hamiltonian in \eqref{eq:string_hamiltonian} implies that $h_1$ is independent of momentum and $h_2$ is independent of strain.
In the main comparisons, only the constitutive-effort model $\hat h_1$ is modified.
The momentum-effort model $\hat h_2$ retains the baseline formulation throughout, providing a controlled comparison and reflecting the comparatively simple relation $h_2 = p$.
We consider the following settings:

\emph{Baseline:}
Both $\hat h_1$ and $\hat h_2$ use the complete input $(p,q,z)$ and an ARD-RBF covariance.
This configuration does not explicitly impose either state invariance and follows the trivial extension as presented in Section~\ref{subsec:baseline_model}.

\emph{SD-RBF:}
The constitutive-effort model $\hat h_1$ uses the reduced input $(q,z)$ and the product covariance $k_q^{\mathrm{RBF}}(q,q') k_z^{\mathrm{RBF}}(z,z')$.
Removing $p$ imposes the known momentum independence of $h_1$, while retaining its dependence on strain and spatial position.
As discussed in Section~\ref{sec:structure_aware_kernel}, this product is equivalent to an ARD-RBF covariance on $(q,z)$.
The comparison with the baseline therefore isolates the effect of excluding the irrelevant momentum input.

\emph{SD-matched (proposed method):}
The effort model $\hat h_1$ retains the input $(q,z)$ and the RBF strain factor, but uses the covariance $k_q^{\mathrm{RBF}}(q,q') k_z^{\mathrm{matched}}(z,z')$.
Comparison with SD-RBF therefore isolates the additional effect of selecting the spatial covariance according to the qualitative stiffness structure.
SD-matched denotes SD-Per.\ with a periodic spatial kernel for periodic stiffness, SD-Mat.\ with a Mat\'ern-$5/2$ spatial kernel for the $C^2$-bump profile, and SD-CP with a changepoint spatial kernel for the semi-piecewise profile.

The changepoint kernel blends two regional covariance functions using a smooth gate
$
    k_z^{\mathrm{CP}}(z,z') = \bar s(z) \bar s(z') k_{\mathrm{L}}(z,z') + s(z) s(z') k_{\mathrm{R}}(z,z'), 
$
where $s(z) = [1 + \exp(-(z - z_0)/w)]^{-1}$ and $\bar s(z) = 1 - s(z)$.
The kernels $k_{\mathrm{L}}$ and $k_{\mathrm{R}}$ are separately parameterized RBF covariances for the two spatial regimes.
The trainable parameters $z_0$ and $w > 0$ specify the transition location and scale in the standardized spatial coordinate.
This construction represents a smooth transition between distinct covariance structures without prescribing the stiffness function itself.
\begin{rem}
    To satisfy Assumption~\ref{ass:separable_hamiltonian}, the momentum-effort model $\hat{h}_2$ should depend only on $(p,z)$, excluding $q$. However, our numerical experiments indicate that removing $q$ from $\hat{h}_2$ has no noticeable effect on performance.
\end{rem}

\subsection{Training and Reconstruction}
\label{subsec:training}
Both effort models are implemented as exact GPs and trained in full-batch mode by minimizing the negative exact marginal log likelihood with Adam.
Kernel hyperparameters use the default GPyTorch initialization unless otherwise specified.
Training uses double precision and a maximum of $1000$ iterations.
Early stopping is applied, and the checkpoint attaining the lowest training loss is restored.
Table~\ref{tab:training_hyperparameters} lists the learning rate $\eta$, early-stopping patience $P$, and minimum-improvement threshold $\Delta_{\min}$.
The settings are shared across kernel architectures within each stiffness profile.
The constitutive-effort models for the $C^2$-bump and semi-piecewise profiles use smaller learning rates, longer patience, and smaller improvement thresholds to accommodate their convergence behavior.
The momentum-effort model uses the same settings across all profiles.

\begin{table}[b]
    \centering
    \caption{Optimization settings for the effort models, shared across kernel architectures.}
    \label{tab:training_hyperparameters}
    \begin{tabular}{lccc}
        \toprule
        Model and profile & {$\eta$} & {$P$} & {$\Delta_{\min}$} \\
        \midrule
        $\hat h_2$, all profiles & \num{1e-1} & $15$ & \num{1e-3} \\
        $\hat h_1$, periodic and constant & \num{1e-1} & $15$ & \num{1e-3} \\
        $\hat h_1$, $C^2$ bumps & \num{2e-2} & $50$ & \num{1e-4} \\
        $\hat h_1$, semi-piecewise & \num{1e-2} & $50$ & \num{1e-4} \\
        \bottomrule
    \end{tabular}
\end{table}

\begin{table*}[b]
    \centering
    \caption{Average numerical metrics for the considered stiffness profiles and kernel configurations.
    SD-RBF, SD-Per., SD-Mat., and SD-CP denote structural decoupling with RBF, periodic, Mat\'ern, and changepoint spatial kernels, respectively.}
    \label{tab:quantitative_results}
    \setlength{\tabcolsep}{5.0pt}
    \renewcommand{\arraystretch}{1.15}
    \begin{tabular}{ll SSS SSS}
        \toprule
        Stiffness & Kernel & {$\varepsilon_q$} & {$\varepsilon_u$} & {$\varepsilon_E$} & {Curl RMSE} & {Eff. RMSE} & {Force RMSE} \\
        \midrule
        Constant & RBF & \num{4.69e-2} & \num{1.23e-2} & \num{1.15e-3} & \num{2.63e-3} & \num{4.77e-2} & \num{1.86e-1} \\
        \midrule
        Periodic & RBF & \num{5.60} & \num{1.18} & \num{1.11e4} & \num{8.19e-1} & \num{1.76} & \num{8.82} \\
        Periodic & SD-RBF & \num{1.73} & \num{6.65e-1} & \num{5.01e1} & \num{3.56e-3} & \num{1.50} & \num{6.30} \\
        \textbf{Periodic} & \textbf{SD-Per.} & \textbf{\num{8.14e-2}} & \textbf{\num{2.06e-2}} & \textbf{\num{3.52e-3}} & \textbf{\num{3.56e-3}} & \textbf{\num{6.03e-2}} & \textbf{\num{3.08e-1}} \\
        \midrule
        $C^2$ bumps & RBF & \num{5.47} & \num{1.18} & \num{1.64e3} & \num{4.75e-1} & \num{1.45} & \num{6.97} \\
        $C^2$ bumps & SD-RBF & \num{1.27} & \num{7.13e-1} & \num{5.10} & \num{3.72e-3} & \num{1.23} & \num{5.97} \\
        \textbf{$C^2$ bumps} & \textbf{SD-Mat.} & \textbf{\num{3.53e-1}} & \textbf{\num{1.03e-1}} & \textbf{\num{2.23e-2}} & \textbf{\num{3.72e-3}} & \textbf{\num{3.09e-1}} & \textbf{\num{1.45}} \\
        \midrule
        Semi-piecewise & RBF & \num{2.38} & \num{1.01} & \num{2.33e2} & \num{1.93e-1} & \num{1.17} & \num{5.31} \\
        Semi-piecewise & SD-RBF & \num{2.24} & \num{8.37e-1} & \num{1.65e2} & \num{3.58e-3} & \num{1.01} & \num{4.04} \\
        \textbf{Semi-piecewise} & \textbf{SD-CP} & \textbf{\num{3.24e-1}} & \textbf{\num{6.90e-2}} & \textbf{\num{1.26e-2}} & \textbf{\num{3.58e-3}} & \textbf{\num{2.92e-1}} & \textbf{\num{1.06}} \\
        \bottomrule
    \end{tabular}
\end{table*}

During evaluation, the trained GPs predict the effort functions at the current integrated state.
The baseline models use $\hat{h}_1(p,q,z)$ and $\hat{h}_2(p,q,z)$.
Structural decoupling uses $\hat{h}_1(q,z)$ and $\hat{h}_2(p,q,z)$. After transformation back into physical units, the predicted efforts are substituted into the discretized port-Hamiltonian dynamics at interior grid points
\begin{equation}
    \dot q_i = [D_z \hat{h}_2]_i, \qquad \dot p_i = [D_z \hat{h}_1]_i - \gamma \hat{h}_{2,i} + F_{\mathrm{ext}}(z_i,t).
    \label{eq:string_gp_rollout}
\end{equation}
Here, $D_z$ is the discrete spatial derivative operator defined in Section~\ref{sec:dataset}, using the same interior and endpoint stencils as the reference simulation.
The reconstruction enforces $p_{\mathrm{GP}}(0,t) = p_{\mathrm{GP}}(L,t) = 0$. The resulting system of ODEs is integrated using the same RK45 method as the reference simulation, producing $p_{\mathrm{GP}}(z,t)$ and $q_{\mathrm{GP}}(z,t)$.
Each simulation uses the same initial conditions, boundary conditions, external forcing, and evaluation grid as the reference trajectory from which training observations are sampled.
This evaluates trajectory reconstruction from sparse observations of that trajectory rather than generalization to new initial conditions or forcing.
For both reference and predicted trajectories, displacement is reconstructed from strain by numerical integration over the spatial domain $\mathcal{Z}$.

\subsection{Results}
\label{subsec:quantitative_results}
Table~\ref{tab:quantitative_results} summarizes the numerical results.  The relative errors $\varepsilon_q$, $\varepsilon_u$, and $\varepsilon_E$ quantify the overall reconstruction accuracy for strain, displacement, and physical energy, respectively, over the complete trajectory, while the curl RMSE measures violations of the mixed-partial integrability condition between the learned effort functions. The effort RMSE evaluates the local accuracy of the learned constitutive effort $\hat{h}_1$, whereas the force RMSE measures the error after applying the spatial derivative $D_z\hat{h}_1$, which directly affects the reconstructed dynamics. 
As an additional baseline, we consider a constant-stiffness case representing a spatially homogeneous system. As expected, the GP-dPHS model with standard RBF kernel performs well in this setting; however, its performance deteriorates substantially for heterogeneous profiles.

Comparing RBF with SD-RBF examines the effect of removing $p$ from the input to $\hat{h}_1$ while retaining the RBF kernel family. SD-RBF reduces strain, displacement, and energy errors for all three heterogeneous profiles, while also improving effort and force RMSE.
The displacement error decreases by approximately $43.6\%$, $39.6\%$, and $17.1\%$ for the periodic, $C^2$-bump, and semi-piecewise profiles, respectively.
The trajectory reconstruction improvements are therefore more pronounced for the periodic and $C^2$-bump cases, with smaller gains for the semi-piecewise profile.
Despite these improvements, substantial reconstruction errors remain, indicating that removing the irrelevant input alone does not adequately capture the heterogeneous structure.

Structural decoupling also reduces the curl RMSE by one to two orders of magnitude for the stiffness profiles considered in this work.
Since $\hat{h}_2$ is independent of $p$, the cross-derivative $\partial \hat{h}_1 / \partial p$ vanishes identically.
The remaining term $\partial \hat{h}_2 / \partial q$ is not constrained to vanish, so these configurations reduce the integrability residual without enforcing exact integrability.
Within each profile, SD-RBF and SD-matched have the same reported curl RMSE because they share the momentum-effort model and both exclude $p$ from the constitutive-effort model.

Using prior information on the stiffness profile as we propose with the structure-aware kernel (\ref{eq:structure_aware_kernel}) produces substantially greater improvements.
SD-Per., SD-Mat., and SD-CP achieve displacement errors of $2.06 \times 10^{-2}$, $1.03 \times 10^{-1}$, and $6.90 \times 10^{-2}$, respectively.
These values correspond to reductions of approximately $98.3\%$, $91.3\%$, and $93.2\%$ relative to the baseline, and $96.9\%$, $85.6\%$, and $91.8\%$ relative to SD-RBF.
The matched models also attain the lowest strain and energy errors and the lowest effort and force RMSE among the tested configurations for each heterogeneous profile.
Their additional reconstruction improvements occur without a further reduction in curl RMSE, highlighting the contribution of the spatial covariance beyond the imposed momentum invariance.
Examples of the predicted trajectories are shown in Figure~\ref{fig:rollout_grid}, where our structure-aware design (SD-matched) most closely reproduces the actual trajectory.
\begin{figure*}[h]
    \centering
    \includegraphics[width=0.99\textwidth]{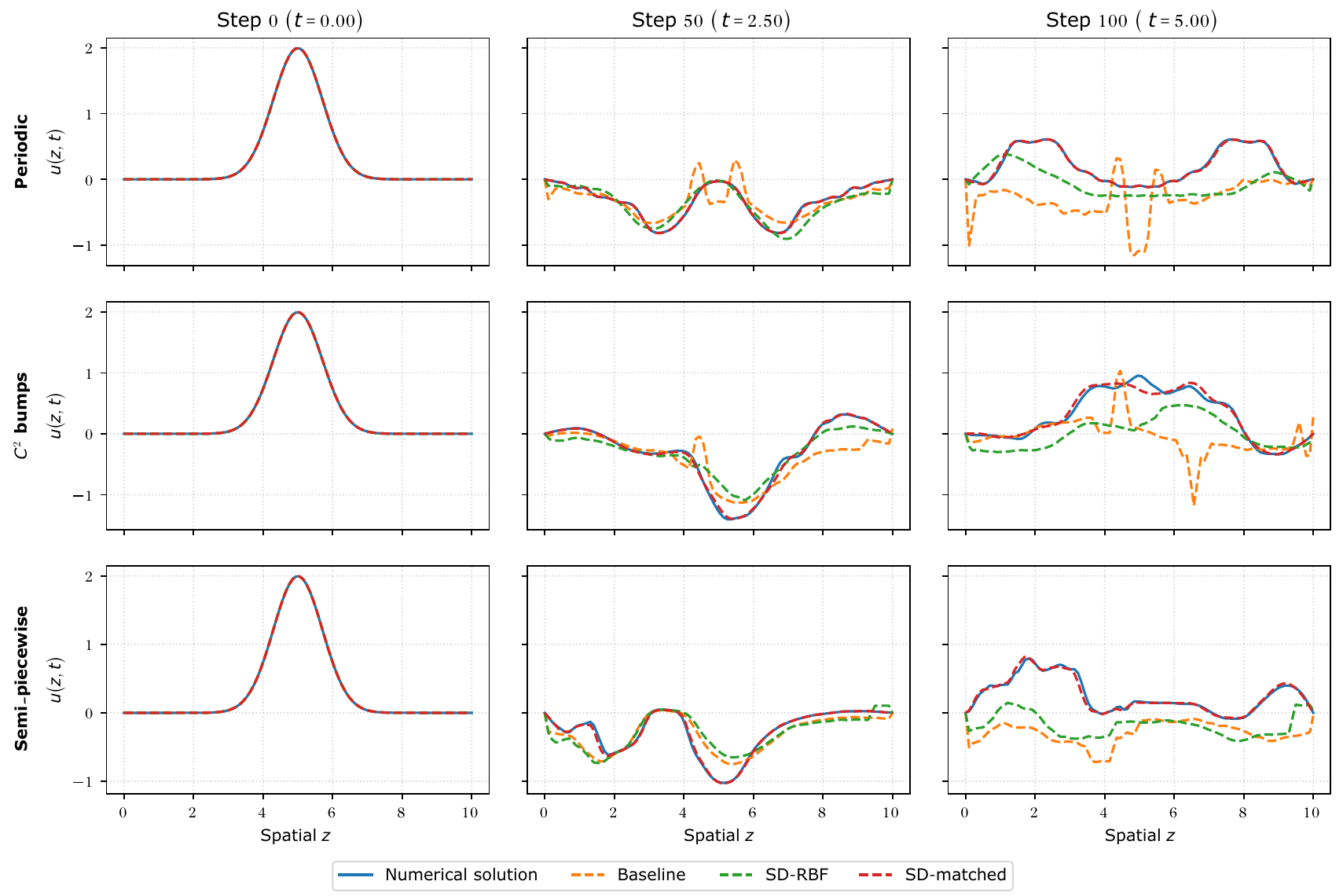}
    \vspace{-0.2cm}\caption{Average reconstructed displacement trajectories for the periodic, $C^2$-bump, and semi-piecewise stiffness profiles.
    Columns correspond to rollout steps 0, 50, and 100.
    Curves compare the numerical reference, stationary RBF baseline, SD-RBF, and profile-matched structural-decoupling model.
    The matched models use periodic, Mat\'ern, and changepoint spatial kernels for the three profiles, respectively.}\vspace{-0.5cm}
    \label{fig:rollout_grid}
\end{figure*}

These experiments demonstrate the benefit of our approach: exploiting known input invariances provides consistent improvements over the baseline, while incorporating profile-matched spatial priors yields further substantial gains. In particular, the best performance is achieved by combining structural decoupling with a spatial covariance adapted to the stiffness profile.

\section{Conclusion} \label{sec:conclusion}
\vspace{-0.1cm}
This work introduced a structure-aware kernel design framework for GP-dPHS of spatially heterogeneous systems, combining known input invariances with a modular separation of state and spatial dependencies.
This enables coarse knowledge of spatial heterogeneity to be incorporated directly into the covariance without requiring an analytical description of the underlying profile.
Numerical experiments demonstrate that structural decoupling consistently improves upon the standard baseline, while substantially larger gains are achieved when the spatial covariance is matched to the qualitative stiffness structure.
These results highlight structure-aware kernel design as an effective mechanism for exploiting spatial prior knowledge in GP-dPHS, with extensions to higher-dimensional and experimental systems providing promising directions for future work.

\bibliographystyle{IEEEtran}
\bibliography{IEEEabrv,refs}


\end{document}